%% file: main.tex
\documentclass{article} 
\usepackage{arxiv}
\input{math_commands.tex}

\usepackage{xcolor}
\definecolor{darkblue}{rgb}{0, 0.12, 0.55}
\definecolor{darkgreen}{rgb}{0, 0.55, 0.12}
\definecolor{darkred}{rgb}{0.6,0,0}
\definecolor{darkgreen}{rgb}{0,0.6,0}
\definecolor{Gray}{gray}{0.9}
\usepackage[breaklinks=true,
            colorlinks,
            linkcolor = darkred,
            urlcolor  = darkblue, 
            citecolor = darkgreen,
            bookmarks = false]{hyperref}
\usepackage{url}

\usepackage{graphicx}
\usepackage{wrapfig}
\usepackage{booktabs}
\usepackage{mathtools}
\usepackage{adjustbox}
\usepackage{caption}
\usepackage{booktabs}
\usepackage{multirow}
\usepackage{xspace}
\usepackage{array}
\usepackage{bm}
\usepackage{bbm}
\usepackage{amsthm,amsmath,amssymb}

\usepackage{listings}
\usepackage{wrapfig}
\usepackage{subcaption}
\usepackage{url}
\usepackage{colortbl}
\usepackage{adjustbox}

\usepackage{paralist}
\usepackage{natbib}

\usepackage{algorithm}
\usepackage{algorithmic}

\definecolor{mymauve}{rgb}{0.58,0,0.82}
\newcommand{\resolved}[1]{}

\newcommand{\com}[1]{}

\newlength{\mysize}

\usepackage[most]{tcolorbox}
\usepackage{xcolor}
\usepackage{fvextra} 
\usepackage{inconsolata} 

\definecolor{promptbg}{RGB}{245,245,245} 
\definecolor{prompttitlebg}{RGB}{70,70,70} 
\definecolor{prompttitlefg}{RGB}{255,255,255} 

\newtcolorbox{prompttitle}{
    colback=prompttitlebg,
    coltext=prompttitlefg,
    left=6pt, right=6pt, top=6pt, bottom=6pt,
    boxrule=0pt,
    enhanced,
    sharp corners,
}

\newtcolorbox{promptbox}{
    breakable,
    enhanced,
    colback=promptbg,
    colframe=black,
    boxrule=0.5pt,
    sharp corners,
    left=10pt,
    right=10pt,
    top=8pt,
    bottom=8pt,
    fontupper=\ttfamily\small,
}

\DefineVerbatimEnvironment{PromptVerb}{Verbatim}{
    breaklines=true,
    breakanywhere=true,
    breaksymbol={},       
    fontsize=\small,
}

\newtcolorbox{promptblock}[1]{%
    breakable,
    enhanced,
    colback=promptbg,            
    colframe=black,              
    colbacktitle=prompttitlebg,  
    coltitle=prompttitlefg,      
    title=#1,                    
    fonttitle=\bfseries\normalsize,   
    sharp corners,
    boxrule=0.5pt,
    left=10pt,
    right=10pt,
    top=8pt,
    bottom=8pt,
    fontupper=\ttfamily\small,   
}

\input{macro-math.tex}
\title{Walk the Talk: Bridging the Reasoning-Action Gap for Thinking with Images via Multimodal Agentic Policy Optimization}

\author{Wenhao Yang$^{1,2,3}$ \ \ Yu Xia$^{3}$ \ \   Jinlong Huang$^{3,4}$ \ \   Shiyin Lu$^{3}$ \ \   Qing-Guo Chen$^{3}$ \ \  Zhao Xu$^{3}$ \ \ Weihua Luo$^{3}$ \\
\textbf{Kaifu Zhang}$^{3}$ \ \ \textbf{Yuchen Zhou}$^{5,8}$ \ \ \textbf{Xiaobo Xia}$^{6}$ \ \ \textbf{Yuanyu Wan}$^{7}$ \ \ \textbf{Lijun Zhang}$^{1,2}$ \ \ \textbf{Tat-Seng Chua}$^{8}$ \\
$^1$National Key Laboratory for Novel Software Technology, Nanjing University, China \\
$^2$School of Artificial Intelligence, Nanjing University, China \quad 
$^3$ AI Business, Alibaba Group \\
$^4$ School of Automation and Intelligent Sensing, Shanghai Jiao Tong University, China \\
$^5$ School of Intelligent Systems Engineering, Sun Yat-sen University, Shenzhen, China \\
$^6$ School of Information Science and Technology, University of Science and Technology of China \\
$^7$ School of Software Technology, Zhejiang University, China \\
$^8$ School of Computing, National University of Singapore, Singapore, Singapore
}

\begin{document}

\maketitle

\begin{abstract}
Recent advancements in Multimodal Large Language Models (MLLMs) have incentivized models to ``\emph{think with images}'' by actively invoking visual tools during multi-turn reasoning. The common Reinforcement Learning (RL) practice of relying on outcome-based rewards ignores the fact that \emph{textual plausibility often masks executive failure}, meaning that models may exhibit intuitive textual reasoning while executing imprecise or irrelevant visual actions within their agentic reasoning trajectories. This reasoning-action discrepancy introduces noise that accumulates throughout the multi-turn reasoning process, severely degrading the model's multimodal reasoning capabilities and potentially leading to training collapse. In this paper, we introduce \textbf{M}ultimodal \textbf{A}gentic \textbf{P}olicy \textbf{O}ptimization (MAPO), bridging the gap between textual reasoning and visual actions generated by models within their Multimodal Chain-of-Thought (MCoT). Specifically, MAPO mandates the model to generate explicit textual descriptions for the visual content obtained via tool usage. We then employ a novel advantage estimation that couples the semantic alignment between these descriptions and the actual observations with the task reward. Theoretical findings are provided to justify the rationale behind MAPO, which inherently reduces the variance of  gradients, and extensive experiments demonstrate that our method achieves superior performance across multiple visual reasoning benchmarks.  
\end{abstract}

\section{Introduction}
\label{sec:intro}
The evolution of Multimodal Large Language Models (MLLMs) has demonstrated remarkable reasoning capabilities in visual understanding tasks~\citep{Qwen25-VL,seed15vl}. These models typically treat visual images as static inputs, processed by a vision encoder before being fed into the LM decoder. Consequently, the reasoning process remains heavily dominated by the language modality. To break away from this \emph{text-only} reasoning paradigm, recent studies have sought to incorporate visual information into the Chain-of-Thought (CoT) produced by models~\citep{NeurIPS:2024:Hu}. This trajectory is exemplified by frontier models like OpenAI o3/o4-mini~\citep{openaio3}, which pioneer the capability to ``\textit{think with images}'' within their internal reasoning chains. Instead of just seeing the images, the o3 model can automatically integrates visual tools, such as zooming and cropping, to analyze them more thoroughly. 

Recent methodologies for replicating o3-style ``thinking with images'' capabilities generally adopt an agentic framework, wherein MLLMs are trained as agents to engage in multi-turn reasoning~\citep{arXiv:2025:DeepEyes,arXiv:2025:Thyme,arXiv:2025:Minio3,arXiv:2025:DRIM}. As illustrated in the left panel of Figure~\ref{fig:intro}, given an input image and a question prompt, the MLLM agent follows a specific chat template to generate a textual reasoning (e.g., \textit{to identify the color, I need a closer look at the box}) accompanied by a visual action (e.g., \textit{\textless tool\_call\textgreater  name:zoom\_in\_tool, arguments:...\textless /tool\_call\textgreater}). Subsequently, the tool is executed to process the image, yielding an observation. This observation is iteratively fed back into the agent's context, driving a multi-turn reasoning loop until the final answer is derived. 

\begin{figure*}[t]
\centering
\includegraphics[width=.95\textwidth]{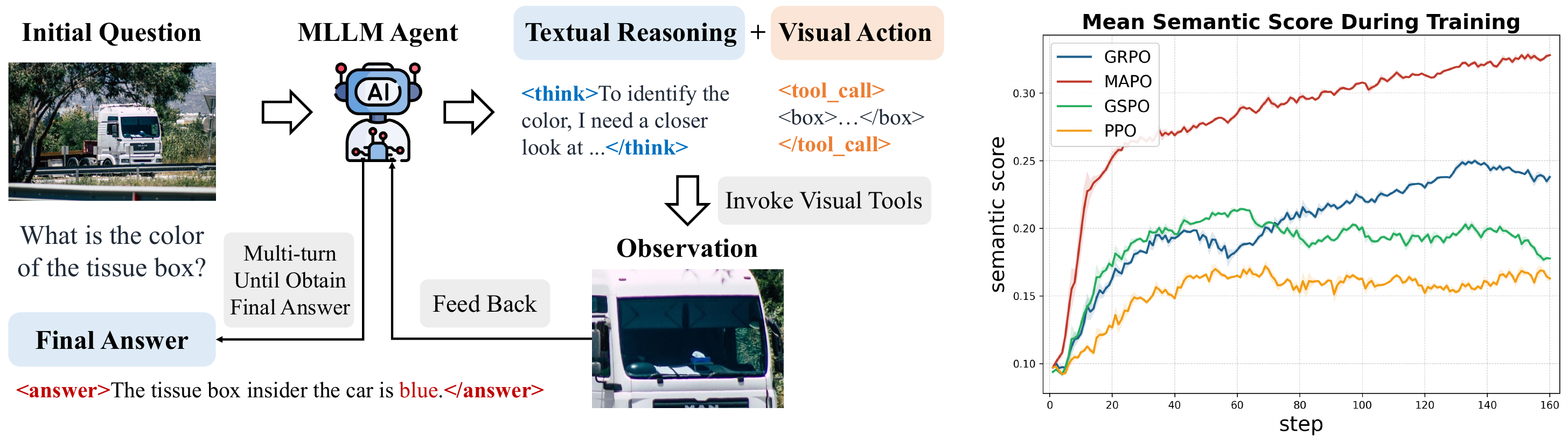}
\caption{\textbf{Left}: Overview of the agentic ``thinking with images'' framework where the MLLM agent executes visual actions guided by textual reasoning. \textbf{Right}: Training curves of the semantic score demonstrating that MAPO significantly outperforms other RL methods. }
\label{fig:intro}
\end{figure*}

Nowadays, driven by the remarkable efficacy of Reinforcement Learning (RL) in agentic reasoning, the field has converged on a standard training pipeline: cold-start Supervised Fine-Tuning (SFT) followed by end-to-end RL, where SFT is employed to equip the model with fundamental capabilities for tool invocation and multi-turn reasoning. In the RL phase, the model is trained to freely explore the policy space, generating reasoning trajectories. However, a critical limitation persists in these prevailing recipes, which predominantly rely on outcome-based reward signals derived solely from the correctness of the final answer, thereby lacking effective supervision for the intermediate agentic trajectory. Such sparse, outcome-centric signals are insufficient to guarantee \textit{cross-modal consistency} when the model thinks with images in its Multimodal CoT (MCoT). A model may generate a textually plausible rationale (e.g., \textit{claiming to examine a specific object}) while executing an irrelevant or imprecise visual action. If the model fortuitously guesses the correct answer despite this misalignment, the standard RL objective inadvertently reinforces this \emph{spurious correlation}. This \textbf{reasoning-action discrepancy} not only introduces significant noise into gradient estimation but also hinders the model from acquiring truly robust visual reasoning capabilities, as the policy is optimized without verifying whether the visual execution actually supports the textual thought. 

To bridge this reasoning-action gap, we propose \textbf{M}ultimodal \textbf{A}gentic \textbf{P}olicy \textbf{O}ptimization (MAPO), a simple yet effective approach that explicitly aligns visual execution with textual reasoning to ensure semantic consistency. Our approach materializes this alignment through a stealthy supervision signal embedded directly within the reasoning loop. First, we enable the model to ``\textit{talk}'' by prompting it to self-generate a descriptive label for the target visual content. Second, we compel the model to ``\textit{walk the talk}'' by leveraging a lightweight, pre-trained CLIP model to quantify the semantic alignment between the generated description and the actual visual observation. This design is highly computationally efficient, incurring negligible memory overhead. Furthermore, we innovatively incorporate a \emph{trajectory-aware discount factor} into the semantic score. This serves a dual advantage: it mitigates reward hacking by penalizing abnormal behaviors (too long trajectories), while offering a necessary tolerance window for the initial exploration phase (e.g., allowing for trial-and-error in early turns). Finally, these semantic signals are seamlessly integrated into the advantage estimation via group aggregation, guiding the policy update towards trajectories that are not only correct in outcome but also consistent in process. 

Extensive experiments across multiple visual understanding benchmarks demonstrate that MAPO achieves superior performance compared to existing baselines. As illustrated in the right panel of Figure~\ref{fig:intro}, our method yields significant improvements in semantic consistency, consistently outperforming other policy optimization methods like GRPO throughout the training process. This enhanced semantic alignment effectively empowers the model to ``walk the talk'' within MCoT reasoning, thereby achieving superior performance on fine-grained visual tasks. Furthermore, our approach exhibits remarkable scalability and training stability. Unlike standard RL methods that often suffer from collapse during prolonged training, MAPO maintains robust performance gains even over extended training steps, suggesting its potential for scaling up agentic reasoning in more complex, open-ended scenarios. 

Theoretically, we establish that MAPO operates as a \textbf{dual variance reduction} framework, providing a principled explanation for the enhanced training stability and superior performance observed in our experiments. Specifically, our analysis identifies two fundamental sources of gradient variance in multimodal agentic reasoning: (1) \textit{spatial sampling variance}, arising from the inherent complexity and diversity of visual prompts, and (2) \textit{signal variance}, stemming from the sparsity and noise of outcome-based rewards. We prove that the group-based advantage estimation in MAPO acts as a \textit{spatial control variate}, effectively filtering out the aleatoric uncertainty shared among responses conditioned on the same visual input, thereby reducing spatial variance by a factor governed by the intra-group reward correlation. Furthermore, we demonstrate that the semantic scoring mechanism provides a \textit{dense, process-aware supervision signal} that is more tightly coupled with the actual reasoning trajectory than sparse outcome rewards, thereby significantly lowering the conditional variance of the reward signal. Together, these two mechanisms complement each other: the former stabilizes gradient estimation across the prompt distribution, while the latter sharpens the reward signal within each trajectory. This dual reduction jointly yields lower-variance policy gradient estimates, which rigorously accounts for the faster convergence and resistance to training collapse that we observe empirically.

Our main contributions are summarized as follows: 
\begin{compactitem} 
\item \textbf{Methodological Innovation:} We propose Multimodal Agentic Policy Optimization, a simple yet but effective RL method that bridges the reasoning-action gap. We enforce semantic consistency between textual reasoning and visual actions without relying on expensive external annotations. 
\item \textbf{Theoretical Foundation:} We provide a theoretical analysis justifying MAPO's effectiveness. We prove that coupling semantic alignment with outcome-based advantage estimation induces a dual variance reduction effect, inherently stabilizing the policy gradient and accelerating convergence. 
\item \textbf{Empirical Superiority:} Extensive experiments validate that MAPO significantly outperforms state-of-the-art baselines on multiple visual reasoning benchmarks. Crucially, our method demonstrates superior scalability and training stability, effectively preventing collapse during long-horizon optimization. 
\end{compactitem}

\section{Related Work}
\paragraph{Multimodal LLMs.} The landscape of Multimodal Large Language Models (MLLMs) has evolved rapidly, with a surge of powerful open-source models demonstrating remarkable capabilities~\citep{llava-next,Qwen25-VL,Internvl3,Keye-VL,Ovis25,seed15vl}. Early efforts in multimodal learning primarily prioritized cross-modal alignment for perception-level tasks, such as image-text matching exemplified by the CLIP series~\citep{NeurIPS:2021:CLIP,ICML:2022:BLIP,ICML:2023:BLIP2,ICCV:2023:Siglip} and standard image captioning~\citep{Vinyals_2015_CVPR,pmlr-v37-xuc15}. Subsequent research has continuously refined the architectural designs of multimodal models, such as Flamingo~\citep{NeurIPS:2022:Flamingo} and LLaVA~\citep{NeurIPS:2023:LLaVA}. Driven by the breakthrough success of Large Language Models (LLMs)~\citep{GPT3,llama1,InstructGPT}, researchers have increasingly sought to extend the language modality to encompass a broader spectrum of other modalities. At the inception of this phase, many models typically bridge a pre-trained Vision Transformer (ViT) with a LLM via a simple projector, such as a MLP~\citep{LLaVA15,internvl15,qwenvl}. Following this, the field witnessed a proliferation of prominent open-source models dedicated to exploring and scaling up model intelligence. These efforts encompass architectural innovations as well as the substantial scaling of model parameters and data volume, as exemplified by the Qwen~\citep{qwen2vl,Qwen25-VL,Qwen3-VL}, InternVL~\citep{internvl,internvl25,Internvl3,internvl35}, Ovis~\citep{ovis,Ovis25}, and LLaVA series~\citep{Llava-onevision}. To further scale the reasoning capabilities of MLLMs for complex problems, researchers have incorporated the Chain-of-Thought (CoT) paradigm into the post-training stage. This adaptation has strengthened a broad range of multimodal tasks, including visual question answering~\citep{arXiv:2023:Zhang,ACL:2025:He,arXiv:2025:Shen} and mathmatical problem-solving~\citep{Mathvista}. 

\paragraph{Tool-Integrated Agentic Reasoning.} To overcome the limitations of static perception, a growing body of research has empowered LLMs with external tools, framing them as autonomous agents. Pioneering works like Toolformer~\citep{NeurIPS:2023:Toolformer} and ReAct~\citep{ICLR:2023:react} demonstrated the power of interleaving reasoning traces with action execution. In the visual domain, frameworks such as MM-ReAct~\citep{arXiv:2023:Yang} and VisProg~\citep{CVPR:2023:visprog} extend this agentic paradigm by equipping MLLMs with visual APIs (e.g., detection, segmentation). Recent advancements have shifted towards optimizing the policy of these agents via RL algorithms, such as REINFORCE~\citep{reinforce++}, PPO~\citep{PPO}, DAPO~\citep{dapo}, GRPO~\citep{GRPO} and GSPO~\citep{gspo}. Along this line of research, a plethora of tool-augmented agentic systems have emerged, including frontier models like OpenAI o3/o4-mini~\citep{openaio3} and Kimi-K2~\citep{kimi-k2}, web agents~\citep{webwatcher,websailor,webshaper}, and agents for mathematical problem-solving~\citep{mai2025agent,simpletir}. These agentic models have demonstrated exceptional multi-turn reasoning capabilities by leveraging tool invocation, such as web browsing and code execution. 

\paragraph{Thinking with Images.} Most of existing MLLMs typically treat images merely as static inputs, a paradigm termed \textit{thinking about images}, wherein the visual modality encodes the image into feature representations for subsequent reasoning by the language model decoder. Consequently, the thinking about images paradigm is inherently confined to text-only reasoning chains, often struggling to achieve optimal performance on complex visual tasks such as precise spatial manipulation~\citep{CVPR:2023:visprog} and long-horizon planning in interactive environments~\citep{2025explorer,wang2025ragen}. With the success of OpenAI o3~\citep{openaio3}, models have begun to evolve from merely seeing images towards a novel reasoning paradigm known as \textit{thinking with images}, as extensively elaborated in recent survey~\citep{survey:twi}. To integrate visual information into the Chain-of-Thought (CoT) reasoning process, current approaches primarily diverge into two main streams: \textit{Reasoning in Latent Space}~\citep{zhang2025latent,wang2025monet,li2025latent,dong2025interleaved} and \textit{Constructing Multimodal CoT}~\citep{NeurIPS:2024:Hu,arXiv:2025:DeepEyes}. The former advocates for internalizing the reasoning process to perform implicit simulation and deduction within continuous high-dimensional vector spaces, whereas the latter emphasizes explicitly generating or manipulating visual information, treating images as a dynamic scratchpad. 
In this work, we focus on the latter paradigm, specifically the explicit construction of MCoT reasoning trajectories. Pioneering works such as DeepEyes~\citep{arXiv:2025:DeepEyes,arXiv:2025:DeepEyesv2} and Thyme~\citep{arXiv:2025:Thyme} have developed agentic systems featuring interleaved multimodal CoT. The prevailing training pipeline typically comprises a cold-start SFT followed by RL, where the SFT stage instills multi-turn reasoning and tool utilization capabilities, while the RL stage facilitates autonomous exploration. Many subsequent works, such as Chain-of-Focus~\citep{arXiv:2025:Zhang:B}, Pixel Reasoner~\citep{arXiv:2025:Su} and others~\citep{huang2025high,wu2025mmsearch,yang2025visionthink,zhu2025active}, have also made valuable contributions toward eliciting more robust \textit{thinking with images} capabilities. These efforts encompass the design of more efficacious reward signals~\citep{guo2025thinking,arXiv:2025:DRIM}, masking mechanisms for excessively long trajectories~\cite{arXiv:2025:Minio3}, and the adaptation of these paradigms to specialized visual tasks like OCR~\citep{xu2025vacot}. Concurrently, RTWI~\citep{li2026reliable} addresses the ``Noisy Thinking'' problem of error propagation from imperfect mining and reasoning by employing text-centric reliability estimation.

\section{Our Approach}
In this section, we first articulate the problem formulation of this work and the technical challenge. Subsequently, we provide a detailed elaboration of our proposed method, named as \textbf{M}ultimodal \textbf{A}gentic \textbf{P}olicy \textbf{O}ptimization (MAPO).

\subsection{Problem Formulation and Technical Challenge}
Following the agentic RL framework established in prior studies~\citep{arXiv:2025:DeepEyes,arXiv:2025:Minio3,arXiv:2025:DRIM}, we formulate the reasoning process of thinking with images as a Markov Decision Process (MDP), incorporating additional observations derived from tool invocation results. At each step $t$, the state $s_t$ of MCoT is defined as:
\begin{equation*}
    s_t = \left\{ (X_0,I_0),(X_1,I_1),\cdots,(X_t,I_t) \right\},
\end{equation*}
where $(X_0,I_0)$ denotes the original image and question, $\mathbf{X}_{\leq t} = \{X_1,\cdots,X_t\}$ denotes the sequence of textual reasoning and visual action (e.g., crop and zoom-in) before step $t$, and $\mathbf{I}_{\leq t} = \{I_1,\cdots,I_t\}$ denotes the sequence of observations returned by the environment after executing the tool. Given the state $s_t$, the model generates the textual reasoning and visual action until the final answer is obtained. 

In the realm of multimodal agentic systems, a critical challenge is that \textit{sparse reward signals} fail to effectively guide MLLMs towards generating \emph{coherent and high-quality} reasoning trajectories. Most existing approaches aimed at incentivizing ``thinking with images'' lack step-level supervision for intermediate visual actions and textual reasoning, largely because determining the optimal action at each step a priori is intractable. Consequently, these methods predominantly rely on outcome-based reward signals. 

This reliance on sparse feedback often masks executive failures, where the model generates plausible textual rationales while performing imprecise or irrelevant visual actions. We term this misalignment as the \textbf{reasoning-action gap}, which fundamentally hinders the acquisition of robust agency.

\subsection{Trajectory-aware Supervision via Semantic Scoring} 
To bridge the reasoning-action gap, we introduce a mechanism that compels the MLLM agent to explicitly articulate its visual intent and verifies whether its execution aligns with this intent. This process transforms the implicit ``thought'' into an explicit, verifiable ``talking'' and ``walking'' loop. 

Standard methods typically generate a tool call action $a_t$ directly from the reasoning context. However, the semantic intent behind $a_t$ often remains latent and ambiguous. To make this intent verifiable, we mandate the model to generate a concise descriptive label $y_t$ alongside the action, explicitly describing what it expects to see. 

Formally, at each step $t$, the policy $\pi_\theta$ generates $(a_t, y_t) \sim \pi_\theta(\cdot | s_t)$, where $s_t$ denotes the current state. This explicit label serves as the semantic anchor for the subsequent verification. Upon executing action $a_t$, the environment returns a visual observation $I_t$ (e.g., a cropped image). To quantify the consistency between the agent's expectation $y_t$ and the actual reality $I_t$, we employ a lightweight, pre-trained CLIP model as a semantic verifier. The raw semantic score $z_t$ is calculated as the cosine similarity between the embeddings of the label and the image, i.e., 
\begin{equation}\label{eqn:clip}
    z_t = \text{CLIP}(y_t, I_t)
\end{equation}
This score provides a dense supervision signal, quantifying the alignment between the text description and the visual content. Furthermore, this verification incurs negligible computational overhead compared to the MLLM inference. 

\begin{figure*}[t]
\centering
\includegraphics[width=.95\textwidth]{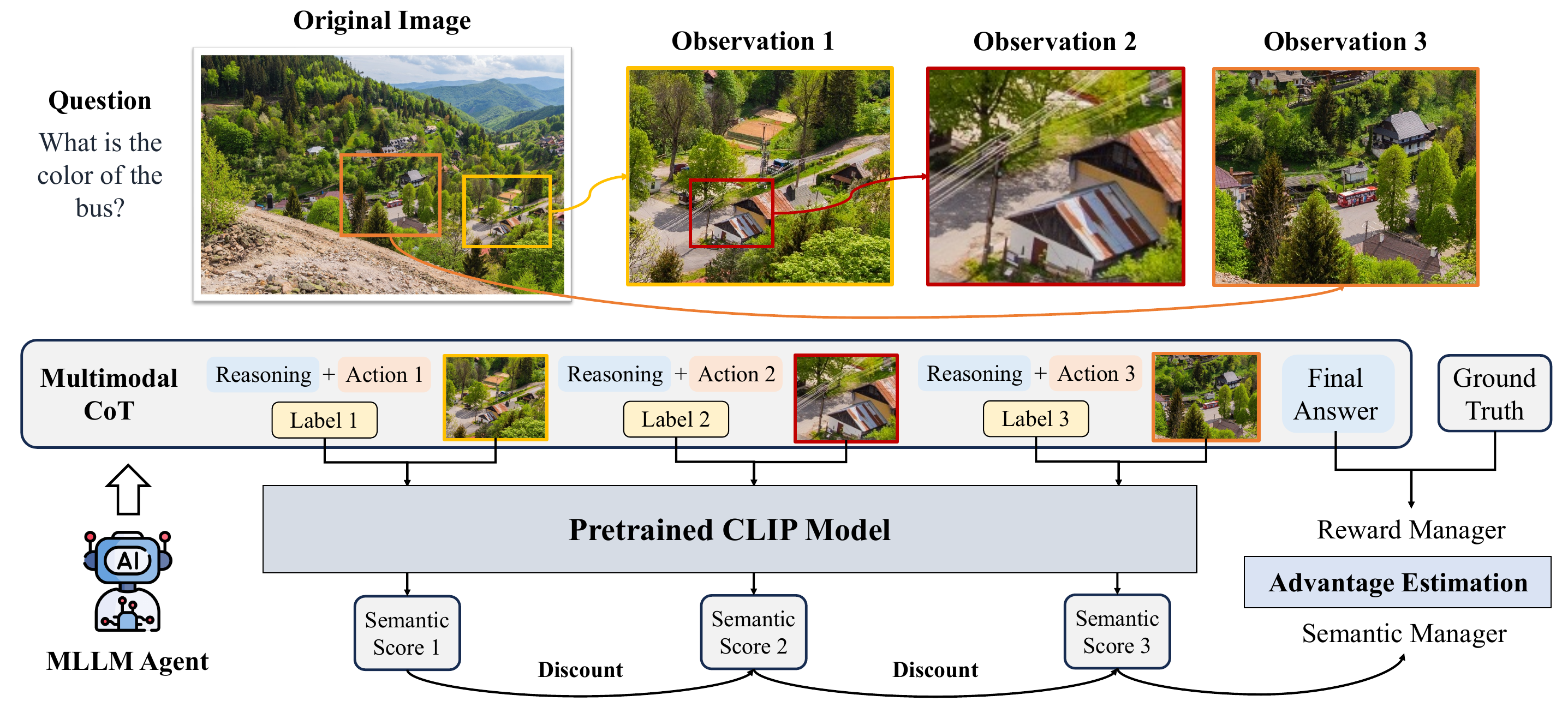}
\caption{\textbf{Overview of MAPO.} Our method bridges the reasoning-action gap by using a CLIP model to measure the semantic alignment between self-generated labels and observations. These signals are integrated into the advantage estimation to enforce process consistency.}
\label{fig:method}
\end{figure*}

While $z_t$ measures local consistency, simply summing these scores may incentivize the agent to prolong the trajectory to accumulate rewards (reward hacking) or prematurely penalize necessary trial-and-error behaviors during early exploration. To address this, we introduce a \emph{trajectory-aware discount factor} $\lambda$ to modulate the semantic score. The final semantic score for the whole trajectory is defined as:
\begin{equation}\label{eqn:sem}
R_{\text{sem}} = \frac{1}{T}\sum_{t=1}^T \lambda^{T-t} z_t,
\end{equation} 
where $T$ is the number of tool calls. It provides a tolerance window in the initial steps to encourage exploration and imposes a stricter penalty (or decay) as the trajectory length~$T$ increases, thereby discouraging inefficient redundancy.

\begin{algorithm}[t]
   \caption{Multimodal Agentic Policy Optimization}
   \label{alg:mapo}
\begin{algorithmic}[1]
\REQUIRE Multimodal Dataset $\mathcal{D}$, MLLM Policy $\pi_\theta$, CLIP Model $\mathcal{M}_{\text{clip}}$, Group Size $G$, Balance Coefficient $\beta$ 
\STATE Initialize policy parameters $\theta$ 
\WHILE{not converged}
    \STATE Sample a batch of query prompts $Q \sim \mathcal{D}$
    \FOR{each question $q \in Q$}
       \STATE \textcolor{gray}{// Phase 1: Group Sampling and Semantic Scoring}
    \STATE Sample $G$ trajectories $\{\tau_1, \dots, \tau_G\}$ from current policy $\pi_{\theta_{old}}$
        \FOR{each trajectory $\tau_i$ in group}
          \STATE Initialize semantic reward sum $R_{\text{sem}}^{(i)} \leftarrow 0$
            \FOR{step $t = 1, \dots, T$}
              \STATE Generate $(a_t, y_t) \sim \pi_\theta(\cdot | s_t)$
               \STATE Execute $a_t$ to obtain visual observation $I_t$      
               \STATE Calculate step-wise semantic score in \eqref{eqn:clip}
        \ENDFOR
        \STATE \textcolor{gray}{// Phase 2: Signals for Trajectory Supervision}
        \STATE Calculate the semantic score $R_{\text{sem}}^{(i)}$ in \eqref{eqn:sem} for the whole trajectory
        \STATE Obtain binary outcome reward $R_{\text{out}}^{(i)} \in \{0, 1\}$ based on the final answer
      \ENDFOR
    \ENDFOR
\STATE \textcolor{gray}{// Phase 3: Policy Update}
    \STATE Update $\theta$ by maximizing the MAPO objective in \eqref{eqn:MAPO}
\ENDWHILE
\end{algorithmic}
\end{algorithm}

\subsection{Multimodal Agentic Policy Optimization} 
Figure~\ref{fig:method} provides an overview of MAPO. The process commences with the MLLM agent generating a chain of reasoning, where each step explicitly interleaves a textual rationale, a tool-use action, and a descriptive label. To ensure the validity of these actions, a pre-trained CLIP model acts as a semantic verifier, calculating the alignment between the generated label and the actual visual observation. These step-wise semantic scores are subsequently modulated by a trajectory-aware factor and integrated with the task outcome. Specifically, MAPO optimizes the following objective: 
\begin{equation}\label{eqn:MAPO}
\begin{aligned}
    &\mathcal{J}_{\text{MAPO}} (\theta) = \mathbb{E}_{q\sim \mathcal{D},\{\tau_i\}_{i=1}^G \sim \pi_{\theta_{old}}(\cdot | q)} \\
     & \left[ \frac{1}{G} \sum_{i=1}^G\frac{1}{\vert\tau_i\vert} \sum_{\mathcal{I}=1}^{\vert\tau_i\vert} \min \left( w^{(i)}_{\mathcal{I}} (\theta) \widetilde{A}^{(i)}_{\mathcal{I}} ,\text{clip} (w^{(i)}_{\mathcal{I}} (\theta),\epsilon)\widetilde{A}^{(i)}_{\mathcal{I}} \right)  \right]
\end{aligned}
\end{equation}
where $G$ is the number of generated responses to each query $q$ (i.e., group size), $w^{(i)}_{\mathcal{I}}$ is the importance ratio (token-level GRPO or sequence-level GSPO), $\text{clip}(\cdot,\epsilon)$ constrains the value to $[1-\epsilon,1+\epsilon]$, and advantage $\widetilde{A}^{(i)}_{\mathcal{I}}$ is defined as: 
\begin{equation*}
        \widetilde{A}^{(i)}_{\mathcal{I}} =  \widetilde{A}^{(i)} = \widehat{A}^{(i)}_{\text{out}} + \beta * \widehat{A}^{(i)}_{\text{sem}}, 
\end{equation*}
where $\widehat{A}^{(i)}_{\text{out}}$ and $\widehat{A}^{(i)}_{\text{sem}}$ are the group normalization of the outcome reward $R^{(i)}_{\text{out}}$ and the semantic score $R^{(i)}_{\text{sem}}$, respectively. 

The proposed MAPO is summarized in Algorithm~\ref{alg:mapo}. The training procedure proceeds in three distinct phases to iteratively refine the policy. First, in the Group Sampling and Semantic Scoring phase (Lines 4-13), the model samples a group of $G$ trajectories for each query. Crucially, during the generation of each step $t$, the policy is mandated to explicitly output a descriptive label $y_t$ alongside the action $a_t$, which is immediately scrutinized against the real-time visual observation $I_t$ via the CLIP verifier in~\eqref{eqn:clip}. Second, in the Trajectory Supervision phase (Lines 14-17), these step-wise signals are aggregated using the trajectory-aware discount factor $\lambda$ to construct the holistic semantic score $R_{\text{sem}}$ in~\eqref{eqn:sem}, which is subsequently paired with the binary outcome reward $R_{\text{out}}$. Finally, in the Policy Update phase (Lines 19-20), the model computes the group-normalized advantages $\widetilde{A}^{(i)}$ based on these composite signals and updates the parameters $\theta$ by maximizing the objective $\mathcal{J}_{\text{MAPO}}$ in~\eqref{eqn:MAPO}, thereby reinforcing trajectories that exhibit both high semantic consistency and factual correctness. 

\paragraph{Extension to General Agentic Systems.} 
While this work primarily instantiates MAPO within the context of visual exploration tools (e.g., crop and zoom-in), the underlying framework is inherently model-agnostic and extensible to a broader spectrum of tool-integrated agentic systems. The core philosophy of MAPO, explicitly aligning ``thought'' (reasoning) with ``action'' (execution) via semantic scoring, is not confined to visual perception tasks. By simply adapting the prompt template to mandate a descriptive intent label before any tool invocation, MAPO can be generalized to complex environments involving APIs, code interpreters, or web browsing. For instance, in a web-browsing agent, the model could be prompted to describe the expected content of a webpage before clicking a link; a text-matching model (similar to CLIP) could then verify the semantic alignment between the expectation and the rendered page. Thus, MAPO offers a versatile paradigm for enforcing process consistency across diverse multimodal agentic applications.

\paragraph{Cognitive Grounding.} Conceptually, our idea draws inspiration from the cognitive principle of the unity of knowing and doing. In human visual thinking, genuine understanding is manifested when perceptual actions are strictly governed by semantic intent. We argue that for MLLMs to evolve towards human-level intelligence, achieving this cognitive consistency is not merely optional but imperative. By enforcing the model's visual execution to faithfully reflect its internal reasoning, MAPO aims to transcend mere trajectory optimization and truly unlock the reasoning boundaries of multimodal intelligence.

\paragraph{Implementation Details.} We implement our method based on verl~\citep{verl} that supports visual tool invocation. Regarding the hyperparameter configuration for MAPO, we set $\lambda=0.95$ and $\beta=0.4$. Additionally, we impose a length constraint on the model-generated descriptive labels to prevent reward hacking.   More training details during the RL stage are deferred to Appendix~\ref{app:train}. 

\section{Theoretical Analysis}\label{sec:theory}
In this section, we provide a theoretical justification for MAPO. We identify that the primary challenge in multimodal agentic reasoning is the high variance in gradient estimation during the long-horizon trajectory, which stems from two sources: (1) \textit{spatial sampling variance} due to the complexity of visual prompts, and (2) \textit{signal variance} due to the sparsity and noise of outcome-based rewards.

We formally show that MAPO achieves superior stability through a \textbf{Dual Variance Reduction} mechanism, effectively acting as a spatial analogue to momentum-based optimization while simultaneously denoising the reward signal.

\subsection{Spatial Variance Reduction via Group Aggregation}

Standard RL methods often suffer from high variance when estimating the policy gradient $\hat{g} = \nabla_\theta \log \pi_\theta(y|x) \cdot r(y)$ using mini-batch samples as an approximation of the true gradient $\nabla J(\theta)$. MAPO employs a group of $G$ outputs $\{y_1, \dots, y_G\}$ generated from the same query prompt $q$. The advantage is estimated as $\hat{A}_i = r_i - \bar{r}$, where $\bar{r} = \frac{1}{G}\sum_{j} r_j$.

\begin{proposition}[Spatial Variance Reduction]
\label{prop:spatial_variance}
Let $\rho$ be the correlation coefficient between the rewards of any two responses $y_i, y_j$ generated from the same prompt $q$. The variance of the group-based gradient estimator satisfies:
\begin{equation*}
    \textnormal{Var}(\hat{g}_{\textnormal{group}}) \approx (1 - \rho) \cdot \textnormal{Var}(\hat{g}_{\textnormal{std}}) + \mathcal{O}\left(\frac{1}{G}\right)
\end{equation*}
\end{proposition}

\paragraph{Remark}
The scaling factor $(1 - \rho)$ serves as a quantitative measure of stability gain. In the context of MLLMs, the intra-group correlation $\rho$ is typically high because all candidates $\{y_i\}_{i=1}^G$ are conditioned on the exact same visual input and query. Consequently, the group mean $\bar{r}$ acts as a robust \textit{spatial control variate}. By subtracting this baseline, MAPO effectively filters out the \textit{aleatoric uncertainty} associated with the inherent difficulty of a specific visual prompt (which affects all group members equally), allowing the gradient to focus purely on the \textit{epistemic variations} arising from the  reasoning trajectory. This can be viewed as a ``spatial'' counterpart to momentum-based optimization, stabilizing the update direction across the prompt distribution. 

\begin{proof}[Proof Sketch]
The group mean $\bar{r}$ serves as a \textit{control variate}. Since each sample rollout and the group baseline $\bar{r}$ share the exact same visual context $\mathcal{I}$ and prompt $x$, they are highly correlated (captured by $\rho$). Leveraging the variance decomposition $\text{Var}(X-Y) = \text{Var}(X) + \text{Var}(Y) - 2\text{Cov}(X,Y)$, the strong covariance term significantly negates the variance inherent to the prompt difficulty.
This mechanism is theoretically isomorphic to the variance reduction technique in STORM optimizer~\citep{NeurIPS:2019:STORM}, where historical gradients serve as temporal control variates. Here, MAPO implements a \textit{spatial} momentum via group aggregation.
\end{proof}

\subsection{Signal Variance Reduction via Semantic Scoring}

While group-based methods like GRPO reduce sampling variance, they typically relies on binary outcome rewards $r_{\text{out}}$, which are noisy proxies for the true reasoning quality $Q^*(x, z)$. We demonstrate that incorporating the Semantic Score ($r_{\text{sem}}$) is necessary to mitigate signal noise.

\begin{proposition}[Signal Variance Reduction]
\label{prop:signal_variance}
Let $\hat{g}_{\textnormal{sem}}$ be the gradient estimator using the dense semantic score $r_{\textnormal{sem}}$. By the Law of Total Variance, $\textnormal{Var}(\hat{g}) = \mathbb{E}_{\tau}[\textnormal{Var}_{r}(\hat{g}|\tau)] + \textnormal{Var}_{\tau}(\mathbb{E}_{r}[\hat{g}|\tau])$.
Since $r_{\textnormal{sem}}$ provides dense supervision grounded in the reasoning process $z$, its conditional variance is strictly lower than that of the sparse outcome reward: $\textnormal{Var}(r_{\textnormal{sem}}|\tau) \ll \textnormal{Var}(r_{\textnormal{out}}|\tau)$. Consequently:
\begin{equation*}
    \textnormal{Var}(\hat{g}_{\textnormal{sem}}) < \textnormal{Var}(\hat{g}_{\textnormal{out}})
\end{equation*}
\end{proposition}

\paragraph{Remark} 
Proposition \ref{prop:signal_variance} indicates that the semantic score acts as a low-pass filter, smoothing out the high-frequency noise inherent in sparse outcome rewards. 

\subsection{Convergence Rate Analysis}

Finally, we connect the reduced variance to the convergence properties of the optimization process. We adopt standard assumptions for non-convex optimization~\citep{NeurIPS:2019:STORM}: (1) $L$-smoothness of the objective $J(\theta)$, and (2) bounded variance of the gradient estimator $\sigma^2$.

\begin{theorem}[Convergence Bound of MAPO]
\label{thm:convergence}
Let $\Delta = J(\theta^*) - J(\theta_0)$ be the initial optimality gap. For SGD optimization with a decaying step size over $T$ iterations, the expected gradient norm satisfies:
\begin{equation}
    \min_{0 \le t < T} \mathbb{E}\left[ \| \nabla J(\theta_t) \|^2 \right] \le \frac{C_1 \cdot L \Delta}{\sqrt{T}} + \frac{C_2 \cdot L \sigma^2_{\textnormal{MAPO}}}{\sqrt{T}}
\end{equation}
where $C_1, C_2$ are constants, and $\sigma_{\textnormal{MAPO}}$ is the standard deviation of the MAPO gradient estimator.
\end{theorem}

\paragraph{Remark}
Theorem \ref{thm:convergence} explicitly links the convergence rate to the gradient noise $\sigma_{\text{MAPO}}$. Combining our results from Propositions \ref{prop:spatial_variance} and \ref{prop:signal_variance}, we have established that:
\begin{equation*}
    \sigma_{\text{MAPO}}^2 \approx \underbrace{(1-\rho)\sigma_{base}^2}_{\text{Spatial Reduction}} - \underbrace{\Delta_{\text{sem}}}_{\text{Signal Denoising}} \ll \sigma_{\text{GRPO}}^2
\end{equation*}
This inequality implies that MAPO achieves a tighter error bound than standard GRPO or PPO. By minimizing both spatial and signal variance, MAPO ensures more stable and efficient convergence in the high-dimensional landscape of multimodal agentic reasoning. 

\begin{table*}[t]
\centering
\caption{\textbf{Main Results on real-world fine-grained visual reasoning benchmarks}. The top two outcomes are \textbf{bolded} and \underline{underlined}, respectively. Here, $^*$ denotes results reported in the original paper, and $^{\dagger}$ denotes results reproduced by ourselves through available model weights. }
\label{tab:main_table}
\setlength\tabcolsep{9pt}
\renewcommand{\arraystretch}{1.1}
\begin{tabular}{p{3cm}ccccccc}
\toprule
\multirow{2}[2]{*}{\textbf{Method}} & \multicolumn{3}{c}{\textbf{V*}} & \multicolumn{3}{c}{\textbf{HR-Bench}} & \multirow{2}[2]{*}{\textbf{MME-Real-Lite}} \\
\cmidrule(lr){2-4} \cmidrule(lr){5-7}
& attribute & relative & overall & 4K & 8K & overall & \\
\midrule
GPT-5~\citep{GPT5}  & 70.0 & 73.4 & 71.3 & 74.6 & 72.0 & 73.3 & 55.3 \\
Gemini 2.5 Pro~\citep{gemini25} & 85.5 & 75.2 & 81.4 & 67.3 & 61.4 & 64.3 & 49.2 \\
\midrule
DeepEyes$^{\dagger}$~\citep{arXiv:2025:DeepEyes}  & \underline{89.6} & 85.5  & 88.0	& 74.9	& 71.2 & 73.0 & 48.4 \\
Thyme$^*$~\citep{arXiv:2025:Thyme} & 83.5 & 80.3 & 82.2 & 77.0 & 72.0 & 74.5 & 55.2 \\
Mini-o3$^{\dagger}$~\citep{arXiv:2025:Minio3} & 90.4 & 88.2 & \underline{89.3} & 75.4 & 72.8 & 74.1 & 49.4 \\
DeepEyesv2$^*$~\citep{arXiv:2025:DeepEyesv2} & - & -& 81.8 &77.9 &73.8 &75.9& -   \\
\midrule
\textbf{Ovis2.5-9B}~\citep{Ovis25}  & 81.9 & 81.6 & 81.8 & 74.0 & 68.2 & 71.1 & 46.1  \\
\quad + Coldstart SFT & 73.0 & 84.2 & 77.5 & 74.6 & 71.0 & 72.8 & 47.9 \\
\qquad + PPO & 79.1 & 88.2 & 82.8 & 77.4 & 76.0 & 76.7 & 43.0 \\
\qquad + GRPO~\citep{GRPO}   & 88.0 & 89.5 & 88.6 & \underline{78.5} & \underline{77.0} & \underline{77.8} & \underline{55.5} \\
\qquad + DAPO~\citep{dapo}   & 88.1 & \underline{90.0} & 88.8 & 78.1 & 76.1 & 77.1 & 50.3 \\
\qquad + GSPO~\citep{gspo}    & 87.1 & \textbf{92.0} & 89.1 & 78.3 & 76.5 & 77.4 & 54.9 \\
\rowcolor[RGB]{236,244,252} 
\qquad + MAPO (Ours)   & \textbf{91.1} & 87.2 & \textbf{89.5} & \textbf{81.0} & \textbf{78.6} & \textbf{79.8} & \textbf{55.8} \\

\bottomrule
\end{tabular}
\end{table*}

\section{Experiments}
In this section, we conduct experiments on three benchmarks that require fine-grained visual reasoning capabilities. 

\subsection{Main Results}
\paragraph{Benchmarks.} To evaluate the fine-grained visual reasoning capabilities of the model, we choose three high-resolution benchmarks, including V*~\citep{CVPR:2024:V*}, HR-bench~\citep{AAAI:2025:HRBench}, and MME-Realworld-Lite~\citep{zhang2024mme}. These datasets are characterized by ultra-high-resolution imagery spanning from 2K to 8K, where the questions hinge on minute visual details. This configuration imposes a rigorous test on the model's ability to precisely localize small targets embedded within a vast visual context, thus reflecting the advantages
of the thinking with images paradigm. We present our results in Avg@8 for V* and HR-bench, and Avg@1 for MME-Realworld-Lite. 

\paragraph{Baselines.} Our approach is built upon Ovis2.5-9B~\citep{Ovis25}, and employs open-source datasets~\citep{arXiv:2025:DRIM} to incentivize ``thinking with images'' capability. We compare our method against three types of baselines, including (i) frontier closed-source models: GPT-5~\citep{GPT5} and Gemini 2.5 Pro~\citep{gemini25}; (ii) open-source MLLMs that can think with images: DeepEyes~\citep{arXiv:2025:DeepEyes}, Thyme~\citep{arXiv:2025:Thyme}, Mini-o3~\citep{arXiv:2025:Minio3} and DeepEyesV2~\citep{arXiv:2025:DeepEyesv2}; and (iii) commonly-used RL methods: PPO~\citep{PPO}, GRPO~\citep{GRPO}, DAPO~\citep{dapo} and GSPO~\citep{gspo}. 

\paragraph{Results.} The quantitative results are summarized in Table~\ref{tab:main_table}. Our proposed MAPO demonstrates consistent superiority across all evaluated benchmarks, establishing a new state-of-the-art for open-source agentic MLLMs. 
For open-source multimodal agentic baselines, MAPO consistently exceeds their performance outperforming the previous state-of-the-art Mini-o3 on HR-Bench overall and surpassing DeepEyes across all metrics.  
When benchmarked against other RL algorithms on the same Ovis2.5-9B backbone, MAPO exhibits robust improvements over standard baselines, achieving a significant improvement over the strong GRPO baseline on HR-Bench overall (79.8\% vs. 77.8\%) and maintaining a clear lead on the challenging 8K subset (78.6\% vs. 77.0\%). These results confirm that our method effectively guides the model towards more precise visual exploration compared to purely outcome-based RL. 

\subsection{Key Findings}
In this subsection, we delve into the internal mechanisms of MAPO to validate the effectiveness compared to other methods regarding semantic alignment and training stability.

\begin{figure*}[t]
\centering
\includegraphics[width=.95\textwidth]{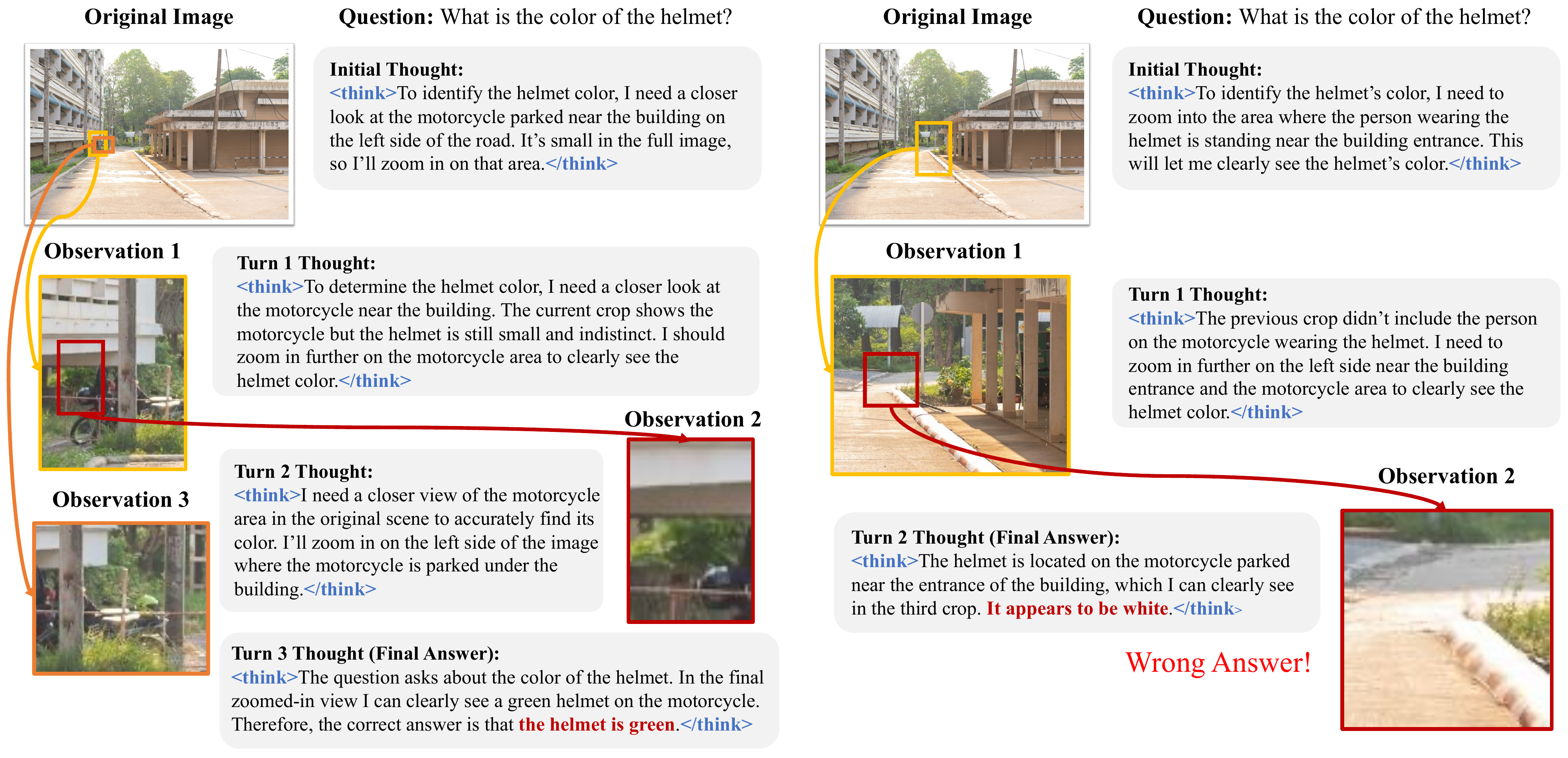}
\caption{\textbf{Visualization of Multimodal Reasoning Trajectories.} We visualize the intermediate reasoning steps of MAPO (\textbf{left}) and GRPO (\textbf{right}) for the same query. While GRPO fails to execute the visual action implied by its text, MAPO successfully aligns the visual execution of MLLM agent with its textual reasoning, effectively bridging the reasoning-action gap.}
\label{fig:walk}
\end{figure*}

\begin{figure}[t] 
    \centering
    \centering
    \includegraphics[width=.5\linewidth]{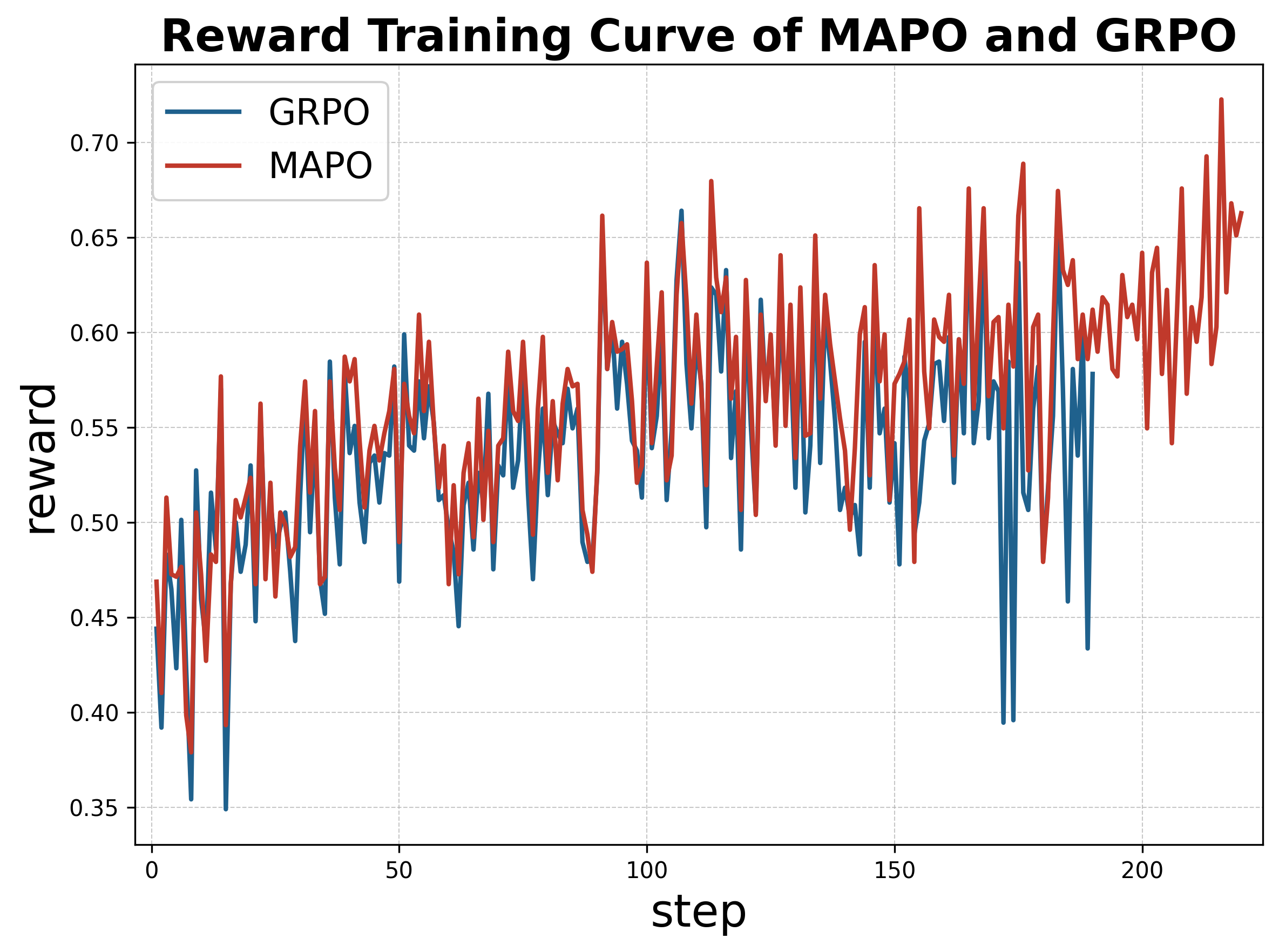} 
    \caption{\textbf{Reward Curves of MAPO and GRPO.} Unlike GRPO, which degrades in later stages, MAPO maintains stable learning, validating its scalability for multimodal agents large-scale training .}
    \label{fig:scalability}
\end{figure}

\paragraph{Visualization of Agentic Trajectories.}
To intuitively understand how MAPO bridges the reasoning-action gap, we visualize the intermediate reasoning and execution trajectories in Figure~\ref{fig:walk}. 
Consider the case where the agent is tasked to identify the color of a small ``helmet''. The baseline GRPO (right) generates a plausible textual plan (``I need to zoom in on the motorcycle...''), but its actual visual action erroneously focuses on an irrelevant region rather than the motorcycle. This misalignment leads to a hallucinated answer (``white'') that contradicts the visual reality.
Conversely, MAPO (left) demonstrates strict semantic consistency. The agent explicitly generates a descriptive label for its intended view, and the CLIP-based feedback guides the policy to execute a precise zoom-in action on the motorcycle. The alignment between the textual reasoning and visual action enables agent to correctly identify the correct answer. 
This visualization underscores that for MLLMs to achieve robust agency, achieving cognitive consistency between "knowing" (textual reasoning) and "doing" (visual action) is imperative, and MAPO effectively grounds the reasoning process in valid visual observations.

\paragraph{Scalability of Agentic Training.} 
One of our motivation for MAPO is addressing the instability inherent in training multimodal agents with sparse rewards. Figure~\ref{fig:scalability} visualizes the reward curves of MAPO compared to the strong baseline GRPO throughout the training phases. 
As observed, GRPO suffers from severe performance oscillation and eventual collapse during the later stages of long-horizon optimization. We observe that at this stage, the model fails to correctly invoke tools, and its performance on benchmarks suffers a catastrophic drop. In contrast, MAPO maintains a robust upward reward curve, demonstrating superior scalability in training large-scale agents in open-ended environments. This phenomenon aligns with our theoretical analysis in Section~\ref{sec:theory}, attributing the failure to the high variance induced by the reasoning-action gap. In contrast, MAPO maintains a robust upward trajectory, yielding continuous performance gains even over extended training steps. By effectively reducing both spatial sampling variance (via group aggregation) and signal variance (via semantic scoring), MAPO demonstrates superior scalability, making it suitable for training large-scale agents in complex, open-ended environments.

\section{Conclusion}
In this paper, we identify the reasoning-action gap as a fundamental obstacle in incentivize MLLMs to ``think with images''. To bridge this gap, we introduce \textbf{M}ultimodal \textbf{A}gentic \textbf{P}olicy \textbf{O}ptimization (MAPO), a novel simple yet but effective framework that enforces semantic consistency between the textual reasoning and visual action. 
By embedding a stealthy, verify-then-reinforce supervision loop, MAPO effectively transforms implicit thoughts into verifiable actions.
Theoretically, we prove that our approach operates as a dual variance reduction mechanism, mitigating both sampling noise and reward sparsity. 
Empirically, MAPO achieves state-of-the-art performance across multiple high-resolution visual reasoning benchmarks. Crucially, it exhibits remarkable stability and scalability compared to standard RL methods, preventing model collapse during long-horizon training. 
We believe MAPO offers a versatile paradigm for grounding agentic reasoning, paving the way for more reliable and autonomous multimodal systems.

\bibliography{conference}
\bibliographystyle{unsrt}

\newpage

\appendix
\begin{center}
    \fontsize{15.5pt}{15.5pt}\selectfont
    \textbf{
        Appendix of ``Walk the Talk: Bridging the Reasoning-Action Gap for Thinking with Images via Multimodal Agentic Policy Optimization''
    }
\end{center}
\let\cleardoublepage\clearpage
\vspace{4pt}

\section{Theoretical Proofs}
\label{sec:appendix_proofs}
In this section, we provide rigorous proofs for the propositions and theorems presented in Section~\ref{sec:theory}. 

\subsection{Proof of Proposition \ref{prop:spatial_variance}}
Let $r_1, \dots, r_G$ be the rewards for $G$ outputs generated from the same prompt $q$. We model these rewards as identically distributed random variables with variance $\text{Var}(r_i) = \sigma^2$.
The advantage estimate for the $i$-th sample in the group-based method  is given by:
\begin{equation*}
    \hat{A}_i = r_i - \bar{r}
\end{equation*}
where $\bar{r} = \frac{1}{G} \sum_{j=1}^G r_j$ is the group mean.

The variance of this advantage estimator is:
\begin{equation}
    \text{Var}(\hat{A}_i) = \text{Var}\left( r_i - \frac{1}{G} \sum_{j=1}^G r_j \right)
\end{equation}
We expand the summation term by separating $r_i$ from the rest of the group:
\begin{equation}
    \hat{A}_i = r_i - \frac{1}{G} r_i - \frac{1}{G} \sum_{j \neq i} r_j = \left(1 - \frac{1}{G}\right) r_i - \frac{1}{G} \sum_{j \neq i} r_j
\end{equation}
Let the covariance between any two distinct rewards in the group be $\text{Cov}(r_i, r_j) = \rho \sigma^2$ for all $i \neq j$, where $\rho$ is the intra-group correlation coefficient.
Using the property of variance for a linear combination of random variables, $\text{Var}(\sum a_k X_k) = \sum a_k^2 \text{Var}(X_k) + \sum_{k \neq l} a_k a_l \text{Cov}(X_k, X_l)$, we calculate:

\begin{align}
    \text{Var}(\hat{A}_i) &= \left(1 - \frac{1}{G}\right)^2 \sigma^2 + \sum_{j \neq i} \left(-\frac{1}{G}\right)^2 \sigma^2 \\
    &\quad + \sum_{j \neq i} 2 \left(1 - \frac{1}{G}\right)\left(-\frac{1}{G}\right) \text{Cov}(r_i, r_j) \\
    &\quad + \sum_{j \neq k, j,k \neq i} \left(-\frac{1}{G}\right)\left(-\frac{1}{G}\right) \text{Cov}(r_j, r_k)
\end{align}

Substituting $\text{Cov}(r_i, r_j) = \rho \sigma^2$ and counting the terms (there are $G-1$ terms in the first sum, $G-1$ covariance pairs involving $i$, and $(G-1)(G-2)$ covariance pairs not involving $i$):

\begin{align}
    \text{Var}(\hat{A}_i) &= \frac{(G-1)^2}{G^2} \sigma^2 + \frac{G-1}{G^2} \sigma^2 \\
    &\quad - \frac{2(G-1)^2}{G^2} \rho \sigma^2 + \frac{(G-1)(G-2)}{G^2} \rho \sigma^2
\end{align}

Simplifying the $\sigma^2$ terms:
\begin{equation}
    \frac{(G-1)^2 + (G-1)}{G^2} \sigma^2 = \frac{(G-1)(G-1+1)}{G^2} \sigma^2 = \frac{G-1}{G} \sigma^2
\end{equation}

Simplifying the $\rho \sigma^2$ terms:
\begin{align}
    \frac{-2(G-1)^2 + (G-1)(G-2)}{G^2} \rho \sigma^2 &= \frac{(G-1) [ -2(G-1) + (G-2) ]}{G^2} \rho \sigma^2 \\
    &= \frac{(G-1) [ -2G + 2 + G - 2 ]}{G^2} \rho \sigma^2 \\
    &= \frac{(G-1)(-G)}{G^2} \rho \sigma^2 = -\frac{G-1}{G} \rho \sigma^2
\end{align}

Combining these results:
\begin{equation}
    \text{Var}(\hat{A}_i) = \frac{G-1}{G} \sigma^2 - \frac{G-1}{G} \rho \sigma^2 = \frac{G-1}{G} (1 - \rho) \sigma^2
\end{equation}

The standard gradient estimator without baseline has variance proportional to $\sigma^2$. The group-based estimator scales this by $(1-\rho)\frac{G-1}{G}$.
Since all outputs share the exact same prompt and visual context, $\rho$ is typically high (positive correlation).
As $G$ becomes sufficiently large, $\frac{G-1}{G} \approx 1$, yielding:
\begin{equation}
    \text{Var}(\hat{g}_{\text{group}}) \approx (1 - \rho) \cdot \text{Var}(\hat{g}_{\text{std}})
\end{equation}
This confirms that the group baseline acts as a control variate, effectively filtering out the variance associated with the prompt difficulty (captured by the shared correlation $\rho$).

\subsection{Proof of Proposition \ref{prop:signal_variance}}
Let $\hat{g}$ be the gradient estimator for a trajectory $\tau$. We compare two estimators: $\hat{g}_{\text{out}}$ (using only sparse outcome reward $r_{\text{out}}$) and $\hat{g}_{\text{sem}}$ (using dense semantic reward $r_{\text{sem}}$).
We model the observed rewards as the true reasoning quality $Q^*(x,z)$ corrupted by noise:
\begin{align}
    r_{\text{out}} &= Q^*(x,z) + \epsilon_{\text{out}}, \quad \epsilon_{\text{out}} \sim \mathcal{N}(0, \sigma_{\text{out}}^2) \\
    r_{\text{sem}} &= Q^*(x,z) + \epsilon_{\text{sem}}, \quad \epsilon_{\text{sem}} \sim \mathcal{N}(0, \sigma_{\text{sem}}^2)
\end{align}
Since $r_{\text{out}}$ is binary (0 or 1) and sparse, its ``noise'' variance $\sigma_{\text{out}}^2$ relative to the continuous reasoning process is high (it treats near-correct and completely wrong reasoning identically if the final answer is wrong). Conversely, $r_{\text{sem}}$ is a continuous, dense signal grounded in the image, so we assume $\sigma_{\text{sem}}^2 < \sigma_{\text{out}}^2$.

The gradient estimator for a fixed trajectory $\tau$ is $\hat{g} = \nabla_\theta \log \pi(\tau) \cdot r$.
Using the Law of Total Variance conditioned on the trajectory $\tau$:
\begin{equation}
    \text{Var}(\hat{g}) = \text{Var}_\tau \left( \mathbb{E}_r [\hat{g} | \tau] \right) + \mathbb{E}_\tau \left[ \text{Var}_r (\hat{g} | \tau) \right]
\end{equation}

\textbf{1. Variance due to Policy Sampling (First Term):}
The term $\mathbb{E}_r [\hat{g} | \tau] = \nabla_\theta \log \pi(\tau) \cdot \mathbb{E}[r|\tau]$ represents the expected gradient direction for a specific trajectory.
Assuming the semantic reward $r_{\text{sem}}$ is calibrated to align with the task objective (i.e., it correlates positively with $r_{\text{out}}$), the expected gradients for both methods point in similar directions: $\mathbb{E}[r_{\text{out}}|\tau] \propto \mathbb{E}[r_{\text{sem}}|\tau]$. Thus, the variance induced by trajectory sampling $\text{Var}_\tau ( \mathbb{E}_r [\hat{g} | \tau] )$ is dominated by the policy's exploration $\pi_\theta$ rather than the reward formulation, making this term comparable for both estimators.

\textbf{2. Variance due to Reward Noise (Second Term):}
The second term, $\mathbb{E}_\tau [ \text{Var}_r (\hat{g} | \tau) ]$, represents the variance due to reward signal noise for a \textit{fixed} trajectory.
\begin{align}
    \text{Var}_r (\hat{g}_{\text{out}} | \tau) &= \| \nabla_\theta \log \pi(\tau) \|^2 \cdot \text{Var}(r_{\text{out}}) = C_\tau \cdot \sigma_{\text{out}}^2 \\
    \text{Var}_r (\hat{g}_{\text{sem}} | \tau) &= \| \nabla_\theta \log \pi(\tau) \|^2 \cdot \text{Var}(r_{\text{sem}}) = C_\tau \cdot \sigma_{\text{sem}}^2
\end{align}
where $C_\tau = \| \nabla_\theta \log \pi(\tau) \|^2$.

Since $\sigma_{\text{sem}}^2 < \sigma_{\text{out}}^2$, it follows that:
\begin{equation}
    \mathbb{E}_\tau [\text{Var}_r(\hat{g}_{\text{sem}} | \tau)] < \mathbb{E}_\tau [\text{Var}_r(\hat{g}_{\text{out}} | \tau)]
\end{equation}

Therefore, the total variance of the MAPO gradient estimator using semantic scoring is strictly lower:
\begin{equation}
    \text{Var}(\hat{g}_{\text{sem}}) < \text{Var}(\hat{g}_{\text{out}})
\end{equation}
This proves that incorporating dense semantic supervision acts as a variance reduction technique for the reward signal.

\paragraph{Remark}
The reduction in variance indicates that $\hat{g}_{\text{sem}}$ serves as a lower-variance proxy for the true policy gradient. In the context of stochastic optimization, lower gradient variance ($\sigma_{\text{sem}}^2 < \sigma_{\text{out}}^2$) directly translates to a tighter bound on the convergence rate (as discussed in Theorem \ref{thm:convergence}). Consequently, MAPO allows for larger effective step sizes and reduces the likelihood of the policy collapsing due to noisy reward spikes during the exploration of the vast multimodal reasoning space.

\subsection{Proof of Theorem \ref{thm:convergence}}
We adopt standard analysis for stochastic gradient ascent on non-convex $L$-smooth functions. 

\textbf{Assumptions:}
\begin{enumerate}
    \item The objective function $J(\theta)$ is $L$-smooth.
    \item The gradient estimator $\hat{g}$ is unbiased: $\mathbb{E}[\hat{g}] = \nabla J(\theta)$.
    \item The gradient estimator has bounded variance: $\text{Var}(\hat{g}) \le \sigma^2$.
\end{enumerate}

The update rule is $\theta_{t+1} = \theta_t + \eta_t \hat{g}_t$ (Gradient Ascent).
From the $L$-smoothness property (Ascent Lemma):
\begin{equation}
    J(\theta_{t+1}) \ge J(\theta_t) + \langle \nabla J(\theta_t), \theta_{t+1} - \theta_t \rangle - \frac{L}{2} \|\theta_{t+1} - \theta_t\|^2
\end{equation}
Substituting the update rule $\theta_{t+1} - \theta_t = \eta_t \hat{g}_t$:
\begin{equation}
    J(\theta_{t+1}) \ge J(\theta_t) + \eta_t \langle \nabla J(\theta_t), \hat{g}_t \rangle - \frac{L \eta_t^2}{2} \|\hat{g}_t\|^2
\end{equation}
Taking expectations with respect to the noise at iteration $t$:
\begin{align}
    \mathbb{E}[J(\theta_{t+1})] &\ge J(\theta_t) + \eta_t \|\nabla J(\theta_t)\|^2 - \frac{L \eta_t^2}{2} \mathbb{E}[\|\hat{g}_t\|^2] \\
    &= J(\theta_t) + \eta_t \|\nabla J(\theta_t)\|^2 - \frac{L \eta_t^2}{2} \left( \|\nabla J(\theta_t)\|^2 + \sigma^2 \right) \\
    &= J(\theta_t) + \left(\eta_t - \frac{L \eta_t^2}{2}\right) \|\nabla J(\theta_t)\|^2 - \frac{L \eta_t^2}{2} \sigma^2
\end{align}
Rearranging to bound the gradient norm:
\begin{equation}
    \left(\eta_t - \frac{L \eta_t^2}{2}\right) \|\nabla J(\theta_t)\|^2 \le \mathbb{E}[J(\theta_{t+1})] - J(\theta_t) + \frac{L \eta_t^2}{2} \sigma^2
\end{equation}
We set the step size $\eta_t = \frac{1}{L\sqrt{T}}$. For sufficiently large $T$, $\frac{L \eta_t}{2} \le \frac{1}{2}$, thus $\eta_t - \frac{L \eta_t^2}{2} \ge \frac{\eta_t}{2}$. Summing over $t=0, \dots, T-1$ and taking total expectation:
\begin{equation}
    \frac{\eta_t}{2} \sum_{t=0}^{T-1} \mathbb{E}[\|\nabla J(\theta_t)\|^2] \le J(\theta^*) - J(\theta_0) + \frac{T L \eta_t^2}{2} \sigma^2
\end{equation}
Let $\Delta = J(\theta^*) - J(\theta_0)$. Substituting $\eta_t = \frac{1}{L\sqrt{T}}$:
\begin{equation}
    \frac{1}{2L\sqrt{T}} \sum_{t=0}^{T-1} \mathbb{E}[\|\nabla J(\theta_t)\|^2] \le \Delta + \frac{T L}{2} \left(\frac{1}{L\sqrt{T}}\right)^2 \sigma^2 = \Delta + \frac{\sigma^2}{2L}
\end{equation}
Dividing by $\frac{T}{2L\sqrt{T}} = \frac{\sqrt{T}}{2L}$:
\begin{equation}
    \frac{1}{T} \sum_{t=0}^{T-1} \mathbb{E}[\|\nabla J(\theta_t)\|^2] \le \frac{2L\Delta}{\sqrt{T}} + \frac{\sigma^2}{\sqrt{T}}
\end{equation}
This yields the bound in Theorem \ref{thm:convergence} (where $\sigma = \sigma_{\text{MAPO}}$ implies $\sigma^2$ variance terms):
\begin{equation}
    \min_{t} \mathbb{E}\left[ \| \nabla J(\theta_t) \|^2 \right] \le \frac{C_1 \cdot L \Delta}{\sqrt{T}} + \frac{C_2 \cdot \sigma^2}{\sqrt{T}}
\end{equation}
Since we established in Propositions~\ref{prop:spatial_variance} and~\ref{prop:signal_variance} that $\sigma_{\text{MAPO}}^2 \ll \sigma_{\text{GRPO}}^2$ (due to both spatial correlation $\rho$ and signal denoising via $r_{\text{sem}}$), the constant factor in the convergence bound is significantly reduced for MAPO. This implies that for a fixed number of iterations $T$, MAPO is guaranteed to reach a stationary point with smaller gradient norm (i.e., closer to a local optimum) compared to standard methods.

\section{Prompt Template}
In this section, we provide the system prompt  used during training and evaluation. The proposed template mainly follows prior works~\citep{arXiv:2025:DeepEyes,arXiv:2025:DRIM}. 

\begin{promptblock}{SYSTEM PROMPT}
\begin{PromptVerb}
You are a helpful assistant.

# Context
In each turn, new images might be provided as a result of your tool calls. The images are numbered sequentially starting from 1. You can refer to any image that has appeared so far in the conversation using its `image_idx`.

# Tools
You may call one or more functions to assist with the user query.
You are provided with function signatures within <tools></tools> XML tags:
<tools>
{
  "type":"function",
  "function":{
    "name":"image_zoom_in_tool",
    "description":"Zoom in on a specific region of an image by cropping it. The new cropped image will be available in the next turn.",
    "parameters":{
      "type":"object",
      "properties":{
        "image_idx":{
          "type":"integer",
          "description":"The 1-based index of the image to perform the zoom-in operation on. The available images are provided and numbered in the user's prompt."
        },
        "bbox_2d":{
          "type":"string",
          "description":"The bounding box of the region to zoom in, as a string '<box>(x1,y1),(x2,y2)</box>' in relative coordinates (0.0 to 1.0) for the selected image, where (x1, y1) is the top-left corner and (x2, y2) is the bottom-right corner."
        },
        "label":{
          "type":"string",
          "description":"The name or label of the object in the specified bounding box."
        }
      },
      "required":["image_idx","bbox_2d", "label"]
    }
  }
}
</tools>

# How to call a tool
Return a json object with function name and arguments within <tool_call></tool_call> XML tags:
<tool_call>
{"name": <function-name>, "arguments": <args-json-object>}
</tool_call>
\end{PromptVerb}
\end{promptblock}

\paragraph{Remark} We emphasize that while prior works~\citep{arXiv:2025:DeepEyes,arXiv:2025:DRIM} also include a label field in the tool arguments, they typically treat it as an \textit{optional} or loosely constrained parameter. Under such relaxed conditions, models tend to degenerate into generating highly abstract and uninformative labels (e.g., generic terms like person or object), which fail to reflect specific visual intent. In contrast, we strictly enforce a \textbf{mandatory descriptive constraint} via the system prompt, compelling the model to generate fine-grained attributes (e.g., ``the man in the red shirt''). This specificity is a prerequisite for our CLIP-based semantic scoring to function effectively as a verifier.

\section{Training Details} \label{app:train}
\subsection{Implementation Framework}
We implement MAPO based on the verl framework~\citep{verl}, utilizing its HybridFlow architecture to decouple generation and training. For the rollout phase, we employ vLLM to accelerate inference with the vllm engine backend. The model training utilizes Fully Sharded Data Parallel (FSDP) with optimizer state offloading to manage memory efficiency. We enable gradient checkpointing and use fused kernels (Triton backend) to further optimize computational throughput.

\subsection{Hyperparameters}
We utilize the Ovis2.5-9B model~\citep{Ovis25} as our policy backbone. The training is conducted with a global batch size of 96. We employ the AdamW optimizer with a cosine learning rate decay schedule, starting from a peak learning rate of $1\times 10^{-6}$ after a 2\% warmup period.
The input image resolution is dynamically handled, with a pixel constraint range of $[10^6, 2\times 10^6]$ pixels to support high-resolution visual reasoning.

Regarding the agentic settings, we allow a maximum of 6 interaction turns per query. The specific hyperparameters for the our proposed MAPO algorithm  are detailed in Table~\ref{tab:hyperparams}.

\begin{table}[h]
\centering
\caption{Hyperparameter settings for MAPO training.}
\label{tab:hyperparams}
\begin{tabular}{l|c}
\toprule
\textbf{Hyperparameter} & \textbf{Value} \\
\midrule
\multicolumn{2}{c}{\textit{Optimization}} \\
\midrule
Global Batch Size & 96 \\
Mini Batch Size & 96 \\
Micro Batch Size (per GPU) & 4 \\
Optimizer & AdamW \\
Peak Learning Rate & 1e-6 \\
LR Scheduler & Cosine \\
Warmup Ratio & 0.02 \\
Weight Decay & 0.0 \\
Gradient Clipping & 1.0 \\
\midrule
\multicolumn{2}{c}{\textit{Algorithm}} \\
\midrule
Group Size ($G$) & 8 \\
Clip Ratio ($\epsilon$) & 0.2 \\
KL Coefficient ($\beta_{KL}$) & 0.0 \\
Entropy Coefficient & 0.0 \\
Advantage Estimator & MAPO \\
Semantic Reward Weight ($\beta$) & 0.4 \\
\midrule
\multicolumn{2}{c}{\textit{Generation \& Environment}} \\
\midrule
Sampling Temperature & 0.9 \\
Top-p & 0.99 \\
Top-k & 40 \\
Max Prompt Length & 4500 \\
Max Response Length & 6000 \\
Single Response Limit & 1024 tokens \\
Max Interaction Turns & 6 \\
Min Image Pixels & $1,048,576$ \\
Max Image Pixels & $2,000,000$ \\
\bottomrule
\end{tabular}
\end{table}

\end{document}

%% file: math_commands.tex
\usepackage{amsmath,amsfonts,bm}

\def\eqref#1{equation~\ref{#1}}

\def\1{\bm{1}}

\DeclareMathAlphabet{\mathsfit}{\encodingdefault}{\sfdefault}{m}{sl}
\SetMathAlphabet{\mathsfit}{bold}{\encodingdefault}{\sfdefault}{bx}{n}



%% file: macro-math.tex
\allowdisplaybreaks

\theoremstyle{plain}
\newtheorem{theorem}{Theorem}
\newtheorem{proposition}{Proposition}

\theoremstyle{definition}

\theoremstyle{remark}

